\documentclass{sig-alternate}

    \usepackage{xcolor}
    \usepackage[colorlinks=true,linkcolor=blue]{hyperref}
\usepackage{graphicx}
\usepackage{cite}

\usepackage{amssymb,amsmath}
\DeclareMathOperator*{\diag}{diag}
\DeclareMathOperator{\Tr}{Tr}
\DeclareMathOperator{\Cond}{Cond}
\DeclareMathOperator{\vect}{vec}
\newcommand{\EX}[1]{\mathbb{E}_{X\sim P_{\theta}}\left[#1\right]}
\newcommand{\E}{\mathbb{E}}
\newcommand{\R}{\mathbb{R}}
\newcommand{\B}{\mathcal{B}}
\newcommand{\X}{\mathbb{X}}
\newcommand{\FM}{\mathcal{I}}
\newcommand{\FIM}{\FM_{\theta}}
\newcommand{\levelset}[3]{\{ #3: #1(#3) \leq #1(#2)\}}
\newcommand{\trans}{\mathrm{T}}
\newcommand{\norm}[1]{\lVert #1 \rVert}
\newcommand{\abs}[1]{\lvert #1 \rvert}

\newcommand{\asto}{\stackrel{\mathrm{a.s.}}{\to}}

\renewcommand{\geq}{\geqslant}
\renewcommand{\leq}{\leqslant}

\renewcommand{\epsilon}{\varepsilon}
\renewcommand{\phi}{\varphi}

\newtheorem{theorem}{Theorem}
\newtheorem{proposition}[theorem]{Proposition}

\newdef{remark}{Remark}
\newdef{definition}{Definition}
\newdef{assumption}{Assumption}

\newcommand{\etacmin}{\gamma_{C,\mathrm{min}}}
\newcommand{\etammin}{\gamma_{m,\mathrm{min}}}
\newcommand{\etakmin}{\gamma_{\kappa,\mathrm{min}}}
\newcommand{\etaklim}{\gamma_{\kappa,\mathrm{lim}}}
\newcommand{\etaclim}{\gamma_{C,\mathrm{lim}}}
\newcommand{\muw}{\mu_{w}}

\newcommand{\lln}{LLN}
\providecommand{\mea}{\mu}
\providecommand{\leb}{\mu_\mathrm{Leb}}
\providecommand{\borel}{\mathcal{F}}
\providecommand{\Vf}{V_{f}}
\providecommand{\Jf}{J}
\providecommand{\dd}{\mathrm{d}}
\providecommand{\Vfh}{\widehat{\Vf}}
\newcommand{\bnu}{\Vf}
\newcommand{\hnu}[1]{\Vfh}
\newcommand{\tnabla}{\tilde \nabla}

\begin{document}
%
\conferenceinfo{GECCO'12,} {July 7-11, 2012, Philadelphia, Pennsylvania, USA.}
    \CopyrightYear{2012}
    \crdata{978-1-4503-1177-9/12/07}
    \clubpenalty=10000
    \widowpenalty = 10000

\title{Analysis of a Natural Gradient Algorithm on Monotonic Convex-Quadratic-Composite Functions
}
%
%
%
%
%

\numberofauthors{1} 
%
\author{
%
%
\alignauthor
Youhei Akimoto\\
       \affaddr{Project TAO INRIA Saclay, LRI Universit{\'e} Paris-Sud, 91405 Orsay Cedex, France}\\
       \email{Youhei.Akimoto@lri.fr}
}


\maketitle
\begin{abstract}
In this paper we investigate the convergence properties of a variant of the Covariance Matrix Adaptation Evolution Strategy (CMA-ES). Our study is based on the recent theoretical foundation that the pure rank-$\mu$ update CMA-ES performs the natural gradient descent on the parameter space of Gaussian distributions. We derive a novel variant of the natural gradient method where the parameters of the Gaussian distribution are updated along the natural gradient to improve a newly defined function on the parameter space. We study this algorithm on composites of a monotone function with a convex quadratic function. We prove that our algorithm adapts the covariance matrix so that it becomes proportional to the inverse of the Hessian of the original objective function. We also show the speed of covariance matrix adaptation and the speed of convergence of the parameters. We introduce a stochastic algorithm that approximates the natural gradient with finite samples and present some simulated results to evaluate how precisely the stochastic algorithm approximates the deterministic, ideal one under finite samples and to see how similarly our algorithm and the CMA-ES perform.
\end{abstract}




\newpage
\section*{Errata}
{\small 
{\bf \noindent Errata in the original GECCO'2012 paper (which have been revised in this version)}
\begin{itemize}
 \item Eq.~\eqref{eq:upper-cond} in Theorem~4 \\
{\em In the original paper:}\\
\begin{equation*} 
\Cond(C^{t} A) \leq 1 + \left(\frac{1 - 2\etacmin}{1 - \etacmin}\right)^{t}(\Cond(C^{0} A - 1).
\end{equation*}

{\em In this version:}\\
\begin{equation*}
\Cond(C^{t} A) \leq 1 + (1 - \etacmin)^{t}(\Cond(C^{0} A) - 1). \label{eq:upper-cond}
\end{equation*}

\item The paragraph after \eqref{eq:cond-eq3}\\
{\em In the original paper:} \\
Since
\begin{equation*}
\frac{1 - \eta_{C}^t(\lambda_{1}^{t} + \lambda_{d}^{t})}{(1 - \eta_{C}^t\lambda_{d}^{t})} = \frac{1-\eta_{C}^t \lambda_{1}^{t}(1 + \lambda_{d}^{t}/\lambda_{1}^{t})}{1 - \eta_{C}^t \lambda_{1}^{t} (\lambda_{d}^{t}/\lambda_{1}^{t})} \leq \frac{1 - 2\eta_{C}^t \lambda_{1}^{t}}{1 - \eta_{C}^t \lambda_{1}^{t}},
\end{equation*}
we have
\begin{equation*}
\frac{\Cond(C^{t+1} A) - 1}{\Cond(C^{t} A) - 1} \leq \frac{1 - 2\eta_{C}^t \lambda_{1}^{t}}{1 - \eta_{C}^t \lambda_{1}^{t}}. 
\end{equation*}
Moreover, since the right-hand side of the above inequality is maximized when $\eta_{C}^t \lambda_{1}^{t}$ is minimized and $\eta_{C}^{t}\lambda_{1}^{t}$ is bounded from below by $\etacmin$ because of \eqref{eq:asmc2}, we have
\begin{equation*}
\frac{\Cond(C^{t+1} A) - 1}{\Cond(C^{t} A) - 1} \leq \frac{1 - 2\etacmin}{1 - \etacmin}.
\end{equation*}
This implies $\lim_{t\to\infty}\Cond(C^{t} A) = 1$. The rate of convergence \eqref{eq:rate-cond} and the upper bound \eqref{eq:upper-cond} are immediate consequences of the above inequality. 

{\em In this version:}\\
Since
\begin{equation*}
\frac{1 - \eta_{C}^t(\lambda_{1}^{t} + \lambda_{d}^{t})}{(1 - \eta_{C}^t\lambda_{d}^{t})} = 1 - \frac{\eta_{C}^t \lambda_{1}^{t}}{1 - \eta_{C}^t \lambda_{1}^{t} \Cond^{-1}(C^t A)} \leq 1 - \eta_{C}^t \lambda_{1}^{t}
\end{equation*}
and the right-most side of the above inequality is bounded by $1 - \etacmin$ because of \eqref{eq:asmc2}, we have
\begin{equation*}
\frac{\Cond(C^{t+1} A) - 1}{\Cond(C^{t} A) - 1} \leq 1 - \etacmin. \label{eq:cond-im}
\end{equation*}
This implies $\lim_{t\to\infty}\Cond(C^{t} A) = 1$ and \eqref{eq:upper-cond}. Moreover, since $\lim_{t\to\infty}\Cond(C^{t} A) = \lim_{t\to\infty} \lambda_1^{t}/ \lambda_d^{t} = 1$, we have from \eqref{eq:cond-eq3} that 
\begin{align*}
\limsup_{t \to \infty} \frac{\Cond(C^{t+1} A) - 1}{\Cond(C^{t} A) - 1} &= \limsup_{t\to\infty }\frac{1 - 2 \eta_{C}^{t}\lambda_{1}^{t}}{1 - \eta_{C}^{t}\lambda_{1}^{t}} \\
&\leq \frac{1 - 2\etacmin}{1 - \etacmin}.
\end{align*}
This proves \eqref{eq:rate-cond}.
\end{itemize}
}
\newpage

\section{Introduction}
The Covariance Matrix Adaptation Evolution Strategy (CMA-ES, \cite{Hansen2001ec,Hansen2003ec,Hansen2004ppsn}) is a stochastic search algorithm for non-separable and ill-conditioned black-box continuous optimization. In the CMA-ES, search points are generated from a Gaussian distribution and the mean vector and the covariance matrix of the Gaussian distribution are adapted by using the sampled points and their objective value ranking. These parameters' update rules are designed so as to enhance the probability of generating superior points in the next iteration in a way similar to but slightly different from the (weighted) maximum likelihood estimation. Adaptive-ESs including the CMA-ES are successfully applied in practice. However, their theoretical analysis even on a simple function is complicated and linear convergence has been proven only for simple algorithms compared to the CMA-ES \cite{Auger2005tcs,Jagerskupper2007tcs}.

Resent studies \cite{Akimoto:2010vd,Glasmachers2010gecco} demonstrate the link between the parameter update rules in the CMA-ES and the natural gradient method, the latter of which is the steepest ascent/descent method on a Riemannian manifold and is often employed in machine learning \cite{Peters2008nc,Bhatnagar2009aut,Rattray1998prl,Amari1998nc,Park2000nn,Amari2000nc}. The natural gradient view of the CMA-ES has been developed and extended in \cite{Arnold2011arxiv} and the Information-Geometric Optimization (IGO) algorithm has been introduced as the unified framework of natural gradient based stochastic search algorithms. Given a family of probability distributions parameterized by $\theta \in \Theta$, the IGO transforms the original objective function, $f$, to a fitness function, $\Jf$, defined on $\Theta$. The IGO algorithm performs a natural gradient ascent aiming at maximizing $\Jf$. For the family of Gaussian distributions, the IGO algorithm recovers the pure rank-$\mu$ update CMA-ES \cite{Hansen2003ec}, for the family of Bernoulli distributions, PBIL \cite{Chandy:1995ux} is recovered. The IGO algorithm can be viewed as the deterministic model of a recovered stochastic algorithm in the limit of the number of sample points going to infinity.

The IGO offers a mathematical tool for analyzing the behavior of stochastic algorithms. In this paper, we analyze the behavior of the deterministic model of the pure rank-$\mu$ update CMA-ES, which is slightly different from the IGO algorithm. We are interested in knowing what is the target matrix of the covariance matrix update and how fast the covariance matrix learns the target. The CMA is designed to solve ill-conditioned objective function efficiently by adapting the metric---covariance matrix in the CMA-ES---as well as other variable metric methods such as quasi-Newton methods \cite{NocedalBOOK2006}. Speed of optimization depends on the precision and the speed of metric adaptation to a great extend. There is a lot of empirical evidence that the covariance matrix tends to be proportional to the inverse of the Hessian matrix of the objective function in the CMA-ES. However, it has not been mathematically proven yet. We are also interested in the speed of convergence of the mean vector and the covariance matrix. Convergence of the CMA-ES has not been reported up to this time. We tackle these issues in this work.

In this paper, we derive a novel natural gradient algorithm in a similar way to the IGO algorithm, where the objective function $f$ is transformed to a function $\Jf$ in a different way from the IGO so that we can derive the explicit form of the natural gradient for composite functions of a strictly increasing function and a convex quadratic function. We call the composite functions \textit{monotonic convex-quadratic-composite} functions. The resulting algorithm inherits important properties of the IGO and the CMA-ES, such as invariance under monotone transformation of the objective function and invariance under affine transformation of the search space. We theoretically study this natural gradient method on monotonic convex-quadratic-composite functions. We prove that the covariance matrix adapts to be proportional to the inverse of the Hessian matrix of the objective function. We also investigate the speed of the covariance matrix adaptation and the speed of convergence of the parameters.

The rest of the paper is organized as follows. In Section~\ref{sec:algo} we propose a novel natural gradient method and present a stochastic algorithm that approximates the natural gradient from finite samples. The basic properties of both algorithms are described. In Section~\ref{sec:conv} we study the convergence properties of the deterministic algorithm on monotonic convex-quadratic-composite functions. The convergence of the condition number of the product of the covariance matrix and the Hessian matrix of the objective function to one and its linear convergence are proven. Moreover, the rate of convergence of the parameter is shown. In Section~\ref{sec:exp}, we conduct experiments to see how accurately the stochastic algorithm approximates the deterministic algorithm and to see how similarly our algorithm and the CMA-ES behave on a convex quadratic function. Finally, we summarize and conclude this paper in Section~\ref{sec:conc}.

\section{The algorithms}\label{sec:algo}

We first introduce a generic framework of the natural gradient algorithm that includes the IGO algorithm.

The original objective is to minimize $f: \X \to \R$, where $\X$ is a metric space. Let $\borel$ and $\mea$ be the Borel $\sigma$-field and a measure on $\X$. Hereunder, we assume that $f$ is $\mea$-measurable. Let $\nu$ represent any monotonically increasing set function on $\borel$, i.e., $\nu(A) \leq \nu(B)$ for any $A$, $B \in \borel$ s.t.\ $A \subseteq B$. We transform $f$ to an \textit{invariant cost} function defined as $\Vf: x \mapsto \nu[y: f(y) \leq f(x)]$. Given a family of probability distributions $P_{\theta}$ on $\X$, we define a \textit{quasi}-objective function $\Jf$ on the parameter space $\Theta$ as the expected value of $\Vf$ over $P_{\theta}$, namely
\begin{equation*}
 \Jf(\theta) = \EX{\Vf(X)} \enspace.\label{eq:f}
\end{equation*}

Our algorithm performs the natural gradient descent on a Riemannian manifold $(\Theta, \FIM)$ equipped with the Fisher metric $\FIM$. The Fisher metric is the unique metric that does not depend on the choice of parameterization \cite{Amari2007book}. The natural gradient---the gradient taken w.r.t.\ the Fisher metric---is given by the product of the inverse of the Fisher information matrix $\FIM$ and the ``vanilla'' gradient $\nabla \Jf(\theta)$ of the function. Therefore, the natural gradient of $\Jf$ is $\FIM^{-1} \nabla \Jf(\theta)$ and the parameter update follows
\begin{equation}
\theta^{t+1} = \theta^{t} - \eta^{t} \FM_{\theta^{t}}^{-1} \nabla \Jf(\theta^{t}), \label{eq:ng-f}
\end{equation}
where $\eta^{t}$ is the learning rate.

\subsection{Deterministic NGD Algorithm on $\R^{d}$}

In the following, we focus on the optimization in $\R^{d}$. Thus, $\X = \R^{d}$, $\mea$ is the Lebesgue measure $\leb$ on $\R^{d}$, and $\borel$ is the Borel $\sigma$-field $\B^{d}$ on $\R^{d}$. 

We choose as the sampling distribution the Gaussian $P_{\theta}$ parameterized by $\theta \in \Theta$, where the mean vector $m(\theta)$ is in $\R^{d}$ and the covariance matrix $C(\theta)$ is a symmetric and positive definite matrix of dimension $d$. 

We define the invariant cost $\Vf(x)$ by using the Lebesgue measure $\leb$ as $\Vf(x) = \leb^{2/d}[y: f(y) \leq f(x)]$. Then, the infimum of $\Jf(\theta) = \EX{\Vf(X)}$ is zero located on a boundary of the domain $\Theta$ where the mean vector equals the global minimum of $f$ and the covariance matrix is zero. 

The choice of the parameterization of Gaussian distributions affects the behavior of the natural gradient update \eqref{eq:ng-f} with finite learning rate $\eta^{t}$, although the steepest direction of $\Jf$ on the statistical manifold $\Theta$ is invariant under the choice of parameterization. We choose the mean vector and the covariance matrix of the Gaussian distribution as the parameter as well as are chosen in the CMA-ES and in other algorithms such as EMNA \cite{Larranaga2002book} and cross entropy method \cite{BOER:2005ur}. Let $\theta = [m^\trans, \vect(C)^\trans]^\trans$, where $\vect(C)$ be the vectorization of $C$ such that the ($i$, $j$)th element of $C$ corresponds to $i + d(j-1)$th element of $\vect(C)$ (see \cite{Harville2008book}). Then the Fisher information matrix has an analytical form
\begin{equation}
\FIM = \begin{bmatrix}
C^{-1} & 0 \\ 0 & \frac{1}{2} ( C^{-1} \otimes C^{-1})
\end{bmatrix},\label{eq:fisher}
\end{equation}
where $\otimes$ denotes the Kronecker product operator. Under some regularity conditions for the exchange of integration and differentiation we have
\begin{equation}
\nabla \Jf(\theta) = \EX{\Vf(X) \nabla l(\theta; X)},\label{eq:grad}
\end{equation}
where $l(\theta; x) = \ln p_{\theta}(x)$ is the log-likelihood. The gradient of the log-likelihood $\nabla l(\theta; x)$ can be written as
\begin{equation}
\nabla l(\theta; x) = \begin{bmatrix}
C^{-1}(x - m)\\
\frac{1}{2}\vect (C^{-1}(x - m)(x - m)^\mathrm{T}C^{-1}  - C^{-1})
\end{bmatrix}.\label{eq:l-grad}
\end{equation}
Then, the natural gradient $\FIM^{-1} \nabla \Jf(\theta) = [\delta m^{\trans}, \vect(\delta C)^{\trans}]^{\trans}$ at $\theta = \theta^{t}$ can be written by part
\begin{equation*}
\begin{split}
\delta m^{t} &= \E_{X \sim P_{\theta^{t}}}[\Vf(X) (X - m^{t})] \\
\delta C^{t} &= \E_{X \sim P_{\theta^{t}}}\bigl[ \Vf(X) \bigl((X - m^{t})(X - m^{t})^\trans - C^{t} \bigr) \bigr]
\end{split}
\end{equation*}
With different learning rates for mean vector and covariance matrix updates, the natural gradient descent \eqref{eq:ng-f} reads
\begin{equation}\label{eq:ng}
m^{t+1} = m^{t} - \eta_{m}^{t} \delta m^{t}, \quad 
C^{t+1} = C^{t} - \eta_{C}^{t} \delta C^{t}\enspace.
\end{equation}
We refer to \eqref{eq:ng} for the deterministic natural gradient descent (NGD) method.

\subsection{Stochastic NGD Algorithm on $\R^{d}$}\label{sec:stoc}

When $\nabla \Jf(\theta)$ is not given, we need to estimate the gradient. We approximate the natural gradient and simulate the natural gradient descent as follows. Initialize the mean vector $m^{0}$ and the covariance matrix $C^{0}$ and repeat the following steps until some termination criterion is satisfied:
\begin{itemize}\itemsep0em
\item[1.] Compute the eigenvalue decomposition of $C^{t}$, $[B, D] = \text{eig}(C^{t})$, where $B$ is an orthogonal matrix and $D$ is a diagonal matrix such that $C^{t} = B D B^{\trans}$.
\item[2.] Compute the square root of $C^{t}$, $\sqrt{C^{t}} = B \sqrt{D} B^{\trans}$.
\item[3.] Generate normal random vectors $z_{1}, \dots, z_{n} \sim \mathcal{N}(0, I_{d})$.
\item[4.] Compute $x_{i} = m^{t} + \sqrt{C^{t}} z_{i}$, for $i = 1, \dots, n$.
\item[5.] Evaluate the objective values $f(x_{1}), \dots, f(x_{n})$;
\item[6.] Estimate $\Vf(x_{i})$ as
\begin{equation*}
	\widehat{\Vf}(x_{i}) = \frac{(2 \pi)^{d/2} \det(D)^{1/2}}{n} \sum_{j: f(x_{j}) \leq f(x_{i})} \exp\biggl( \frac{\norm{z_{j}}^{2}}{2} \biggr).
\end{equation*}
\item[7.] Compute the baseline $b = \sum_{i=1}^{n} \widehat{\Vf}(x_{i}) / n$.
\item[8.] Compute the weights $w_{i} = (\widehat{\Vf}(x_{i}) - b) / n$.
\item[9.] Estimate the natural gradient $\delta m^{t}$ and $\delta C^{t}$ as
\begin{equation}\label{eq:ng-est}
\begin{split}
	\widehat{\delta m^{t}} &= \sum_{i=1}^{n} w_{i} (x_{i} - m^{t})\\
	\widehat{\delta C^{t}} &= \sum_{i=1}^{n} w_{i} \bigl((x_{i} - m^{t})(x_{i} - m^{t})^\trans  - C^{t}\bigr)\enspace.
\end{split}
\end{equation}
\item[10.] Compute the learning rates $\eta_{m}^{t}$ and $\eta_{C}^{t}$.
\item[11.] Update the parameters as $m^{t+1} = m^{t} - \eta_{m} \widehat{\delta m^{t}}$ and $C^{t+1} = C^{t} - \eta_{C} \widehat{\delta C^{t}}$.
\end{itemize}
We refer to this algorithm for the stochastic NGD algorithm.

This algorithm generates $n$ samples $x_{i}$ from $\mathcal{N}(m^{t}, C^{t})$ in steps 1--4 and evaluates their objective values in step 5. In step 6, the invariant costs $\Vf(x_{i})$ are evaluated. The estimates $\widehat{\Vf}(x_{i})$ are obtained as follows. By definition we have
\begin{equation*}
\Vf(x) = \biggl( \int \frac{\mathbf{1}_{\{f(y) \leq f(x)\}}}{p_{\theta^{t}}(y)} p_{\theta^{t}}(y) \dd y \biggr)^{2/d}.
\end{equation*}
Applying Monte-Carlo approximation we have 
\begin{equation}
\widehat{\Vf}(x) = \biggl( \frac{1}{n}\sum_{j = 1}^{n} \frac{\mathbf{1}_{\{f(x_{j}) \leq f(x)\}}}{p_{\theta^{t}}(x_{j})}\biggr)^{2/d}.\label{eq:mc-nu}
\end{equation}
Since $p_{\theta^{t}}(x_{j})  = \bigl((2\pi)^{d}\det(D)\bigr)^{-1/2} \exp\bigl( \norm{z_{j}}^{2}/ 2 \bigr)$, we have the estimates $\widehat{\Vf}(x_{i})$ in step 6. Step 7 computes the baseline $b$ that is often introduced to reduce the estimation variance of gradients while adding no bias \cite{Greensmith2004jmlr}. We simply choose the mean value of the $\widehat{\Vf}(x_{i})$ as the baseline. Replacing the expectation in \eqref{eq:ng} with the sample mean and adding the baseline (in step 8) we have the Monte-Carlo estimate of the natural gradient in step 9. Finally in step 11, we update the parameters along the estimated natural gradient with the learning rates computed in step 10. The learning rates are chosen in the following so that they are inverse proportional to the largest eigenvalue of the following matrix
\begin{equation}\label{eq:learning}
Z^{t} = (C^{t})^{-1/2} \widehat{\delta C^{t}} (C^{t})^{-1/2} = \sum_{i=1}^{n} w_{i} (z_{i}z_{i}^\trans  - I_{d}) \enspace.
\end{equation}

\subsection{Difference from the IGO}
The difference between the IGO algorithm and our deterministic algorithm is that the invariant cost in the IGO algorithm is defined by negative of the weighted quantile, $- w(P_{\theta^{t}}[y: f(y) \leq f(x)])$, where $w: [0, 1] \mapsto \R$ is non-increasing weight function. Since the quantile $P_{\theta^{t}}[y: f(y) \leq f(x)]$ depends on the current parameter $\theta^{t}$, $\Vf(x)$ for each $x$ in the IGO algorithm changes from iteration to iteration, whereas it is fixed in our algorithm. This property makes our algorithm easier to analyze mathematically.

The difference between our stochastic algorithm and the pure rank-$\mu$ update CMA-ES \cite{Hansen2003ec} is the same as the difference between the deterministic one and the IGO algorithm. The pure rank-$\mu$ update CMA-ES approximates the quantile value $P_{\theta^{t}}[y: f(y) \leq f(x)]$ by the number of better solutions divided by the number of samples $n$, $R_{i}/n = \abs{\levelset{f}{x_{i}}{x_{j}}}/n$. Therefore, the pure rank-$\mu$ update CMA-ES simulates the same lines as the stochastic NGD algorithm described in Section~\ref{sec:stoc} with the weights $w_{i} = w(R_{i}/n)/n$. 

In Section~\ref{sec:exp} we compare the stochastic NGD algorithm with the pure rank-$\mu$ update CMA-ES where 
\begin{equation}
w_i = \left\{\begin{array}{ll}
1/ \lfloor n/4 \rfloor & \text{if }R_{i}/n \leq \lfloor n/4\rfloor\\
0 & \text{otherwize}.
\end{array}\right.
\label{eq:intermediate}
\end{equation}

\subsection{Basic Properties}

\noindent\textbf{Invariance. }
Our algorithms inherit two important invariance properties from the IGO and the CMA-ES: invariance under monotonic transformation of the objective function and invariance under affine transformation of the search space (with the same transformation of the initial parameters). Invariance under monotonic transformation of the objective function makes the algorithm perform equally on a function $f$ and on any composite function $g \circ f$ where $g$ is any strictly increasing function. For example, the convex sphere function $f(x) = \norm{x}^{2}$ is equivalent to the non-convex function $f(x) = \norm{x}^{1/2}$ for this algorithm, whereas conventional gradient methods, e.g.\ Newton method, assume the convexity of the objective function and require a fine line search to solve non-convex functions. This invariance property is obtained as a result of the transformation $f \mapsto \Vf$. Invariance under affine transformation of the search space is the essence of variable metric methods such as Newton's method. By adapting the covariance matrix, this algorithm attains universal performance on ill-conditioned objective functions. 

\noindent\textbf{Positivity. }
The covariance matrix of the Gaussian distribution must be positive definite and symmetric at each iteration. The next proposition gives the condition on the learning rate $\eta_{C}^{t}$ such that the covariance matrix is always positive definite symmetric. 

\begin{proposition}
Suppose that the learning rate for the covariance update $\eta_{C}^{t} < \lambda_{1}^{-1}(\sqrt{C^{t}}^{-1}\delta C^{t}\sqrt{C^{t}}^{-1})$ in the deterministic NGD algorithm, where $\lambda_{1}(\cdot)$ denotes the largest eigenvalue of the argument matrix. If $C^{0}$ is positive definite symmetric, $C^{t}$ is positive definite symmetric for each $t$. Similarly, if $\eta_{C}^{t} <  \lambda_{1}^{-1}(Z^{t})$ in the stochastic NGD algorithm, where $Z^{t}$ is defined in \eqref{eq:learning}, and if $C^{0}$ is positive definite symmetric, the same result holds.
\end{proposition}
\begin{proof}
Consider the deterministic case \eqref{eq:ng}. Suppose that $C^t$ is positive definite and symmetric. Then,
\begin{equation*}
 C^{t+1} 
 = \sqrt{C^{t}} \left( I_{d} -  \eta_C^{t} \sqrt{C^{t}}^{-1} \delta C^{t} \sqrt{C^{t}}^{-1}\right) \sqrt{C^{t}}.
\end{equation*}
Since $\eta_{C}^{t} < \lambda_{1}^{-1}(\sqrt{C^{t}}^{-1}\delta C^{t} \sqrt{C^{t}}^{-1})$ by the assumption, all the eigenvalues of $\eta_C^{t} \sqrt{C^{t}}^{-1} \delta C^{t} \sqrt{C^{t}}^{-1}$ is smaller than one. Thus, the inside of the brackets is positive definite symmetric and hence $C^{t+1}$ is positive definite symmetric. By mathematical induction, we have that $C^{t}$ is positive definite and symmetric for all $t \geq 0$. The analogous result for the stochastic case is obtained in the same way.
\end{proof}

\noindent\textbf{Consistency. }
The gradient estimator \eqref{eq:ng-est} is not necessarily unbiased, yet it is consistent as is shown in the following proposition. Therefore, one can expect that the stochastic NGD approximates the deterministic NGD well when the sample size $n$ is large. Let $\tnabla: J \mapsto \FIM^{-1} \nabla J$ be the natural gradient operator.

\begin{proposition}\label{prop:consistency}
Let $X_{1}, \dots, X_{n}$ be independent and identically distributed random vectors following $P_{\theta}$. Let $\Vfh(x)$ and $G_{\theta}^{n} = [(\widehat{\delta m^{t}})^{\trans}, \vect(\widehat{\delta C^{t}})^{\trans}]^{\trans}$ be the invariant cost \eqref{eq:mc-nu} and the natural gradient \eqref{eq:ng-est} where $x_{1}, \dots, x_{n}$ are replaced with $X_{1}, \dots, X_{n}$. Suppose that 
\begin{gather}
\E[\Vf(X)^{2}] < \infty. \label{eq:propasm3}
\end{gather}
Then, $G_{\theta}^{n} \asto \tnabla \Jf(\theta)$, where $\asto$ represents almost sure convergence. 
\end{proposition}

\begin{proof}
By the Cauchy-Schwarz inequality we have that $\E[\norm{\Vf(X) \tnabla l(\theta; X)}]^{2} < \E[\Vf(X)^{2}]\E[\norm{\tnabla l(\theta; X)}^{2}]$. Note that $\E[\norm{\tnabla l(\theta; X)}^{2}] = \Tr(\FIM^{-1}) < \infty$. By Jensen's inequality we have that $\E[\Vf(X)]^{2} \leq \E[\Vf(X)^{2}]$. Therefore, \eqref{eq:propasm3} implies
\begin{gather}
\E[\norm{\Vf(X) \tnabla l(\theta; X)}] < \infty \text{ and }\label{eq:propasm1}\\
\E[\Vf(X)]  < \infty \enspace. \label{eq:propasm2}
\end{gather}

Define $h_{n}(x) = \hnu{n}(x) - \bnu(x)$ and decompose $G_{\theta}^{n}$ as
\begin{multline}
G_{\theta}^{n} = \frac{1}{n}\sum_{i=1}^{n} \bnu(X_{i}) \tnabla l(\theta; X_{i}) + \frac{1}{n}\sum_{i=1}^{n} h_{n}(X_{i}) \tnabla l(\theta; X_{i})\\
 - \biggl( \underbrace{\frac{1}{n}\sum_{i=1}^{n} \hnu{n}(X_{i})}_{= b}\biggr) \biggl( \frac{1}{n} \sum_{i=1}^{n} \tnabla l(\theta; X_{i})\biggr).
\label{eq:propconv}
\end{multline}
By \eqref{eq:propasm1} and the strong law of large numbers (\lln), the first summand converges to $\E[\bnu(X) \tnabla l(\theta; X)] = \tnabla \Jf(\theta)$ almost surely as $n \to \infty$. So we have to show that the second term and the third term of \eqref{eq:propconv} converge almost surely to zero.

First, we show the following almost sure convergence
\begin{equation}
\frac{1}{n}\sum_{i=1}^{n} h_{n}(X_{i}) \asto 0 \quad \text{as } n \to \infty. \label{eq:prop1}
\end{equation}
By the definition of $\hnu{n}(x)$, we have 
\begin{align*}
\lim_{n\to\infty}\hnu{n}(x) &= \left(\lim_{n\to\infty} \frac{1}{n}\sum_{j = 1}^{n} \frac{\mathbf{1}_{f(X_{j}) \leq f(x)}}{p_{\theta}(X_{j})}\right)^{2/d}
\intertext{and since \eqref{eq:propasm2} implies $\leb[y: f(y) \leq f(x)] < \infty$ almost everywhere, we have by \lln}
&= \leb^{2/d}[y: f(y) \leq f(x)] = \bnu(x)
\end{align*}
almost surely and almost everywhere in $x$. This implies $h_{n}(x) \asto 0$ almost everywhere in $x$. 

For $m \leq n$, we have
\begin{equation}
\begin{split}
\lim_{n\to\infty} \Bigl\lvert \frac{1}{n}\sum_{i=1}^{n} h_{n}(X_{i}) \Bigr\lvert \leq \lim_{n\to\infty} \frac{1}{n}\sum_{i=1}^{n} \sup_{j \geq n} \abs{h_{j}(X_{i})} \\
\leq \lim_{n\to\infty} \frac{1}{n}\sum_{i=1}^{n} \sup_{j \geq m} \abs{h_{j}(X_{i})}.
\end{split}\label{eq:suph}
\end{equation}
Since $h_{n}(x) \asto 0$ almost everywhere in $x$ as $n \to \infty$, we have $\sup_{j \geq m} \abs{h_{j}(x)} \asto 0$ almost everywhere in $x$ as $m \to \infty$. By the Lebesgue's dominated convergence theorem we have $\E[\sup_{j \geq m} \abs{h_{j}(X)}] \to 0$ as $m \to \infty$. Therefore, we have that $\E[\sup_{j \geq m} \abs{h_{j}(X)}] < \infty$ for $m$ large enough. Then, by applying \lln, we have that the right most side of \eqref{eq:suph} converges to $\E[\sup_{j \geq m} \abs{h_{j}(X)}]$ as $n \to \infty$ and this expectation converges to $0$ as $m \to \infty$. This ends the proof of \eqref{eq:prop1}.

Now we can obtain the almost sure convergence of the third term of \eqref{eq:propconv} to zero. Indeed, the almost sure convergence \eqref{eq:prop1} implies that the limit $\lim_{n\to\infty}\sum_{i=1}^{n} \hnu{n}(X_{i})/n$ agrees with $\lim_{n\to\infty}\sum_{i=1}^{n} \bnu(X_{i})/n$ and we have from \eqref{eq:propasm2} and \lln\ that $\sum_{i=1}^{n} \hnu{n}(X_{i})/n \asto \E[\bnu(X)] < \infty$. Also, by \lln\ we have that $\sum_{i=1}^{n} \tnabla l(\theta; X_{i})/n \asto 0$ as $n \to \infty$. Therefore, the third term of \eqref{eq:propconv} converges to zero almost surely.

To show the convergence of the second term of \eqref{eq:propconv} to zero, we apply the Cauchy-Schwarz inequality to it and we have
\begin{multline*}
\left\lvert \sum_{i=1}^{n} \frac{h_{n}(X_{i}) \tnabla l(\theta; X_{i})}{n} \right\rvert^{2} \leq  \sum_{i=1}^{n} \frac{h_{n}(X_{i})^{2}}{n} \sum_{i=1}^{n} \frac{\norm{\tnabla l(\theta; X_{i})}^{2}}{n}.
\end{multline*}
By \lln\ we have that the second term of the right hand side converges to $\E[\norm{\tnabla l(\theta; X)}^{2}] = \Tr(\FIM^{-1})$. So we have to prove that the first term on the right hand side converges to zero almost surely. The proof of this convergence is done in the same way as above with $h_{n}^{2}$ replacing $h_{n}$.
\end{proof}

We remark that \eqref{eq:propasm2} is the necessary and sufficient condition for $\Jf(\theta)$ to exist and that \eqref{eq:propasm1} is a sufficient condition for the exchange of integral and differentiation used in \eqref{eq:grad}. See e.g.\ \cite[Theorem~16.8]{Billingsley1995book}.

\section{Convergence Properties of the Deterministic NGD Algorithm}\label{sec:conv}

We investigate the convergence properties of the deterministic NGD algorithm \eqref{eq:ng} on a monotonic convex-quadratic composite function $f(x) = g(x^\mathrm{T} A x)$, where $g$ is any strictly increasing function and $A$ is a positive definite symmetric matrix.

\begin{proposition}\label{prop:ng}
The natural gradient can be written as 
\begin{equation*}
\FIM^{-1} \nabla \Jf(\theta) \propto %
\begin{bmatrix}
C A m\\
\vect(C A C)
\end{bmatrix} \enspace.
\end{equation*}
\end{proposition}

\begin{proof}
Since $\leb[y: f(y) \leq f(x)]$ is equivalent to the volume of the ellipsoid $\{y: y^\trans A y \leq x^\trans A x\}$, we have that 
\begin{equation*}
\leb[y: f(y) \leq f(x)] = \frac{2}{\det(A)} V_{d}(\sqrt{x^\mathrm{T} A x}),
\end{equation*}
where $V_{d}(r)$ denotes the volume of the sphere with radius $r$ and is proportional to $r^{d}$. Therefore $\Vf(x) = \leb^{2/d}[y: f(y) \leq f(x)] \propto x^\trans A x$. Since the proportionality constant is independent of $x$, we have
\begin{equation*}
\begin{split}
\Jf(\theta) &= \EX{\Vf(X)}\\
&\propto \E_{X\sim P_{\theta}}\bigl[ X^\trans A X \bigr] = m^\trans A m + \Tr(AC).
\end{split}
\end{equation*}
Differentiating the both side of the above relation, we have
\begin{equation*}
\nabla \Jf(\theta) \propto 
\begin{bmatrix}
2A m\\
\vect(A)
\end{bmatrix}.
\end{equation*}
Premultiplying by $\FIM^{-1}$, we obtain the intended result.
\end{proof}

Now the deterministic NGD algorithm on $g(x^{\trans} A X)$ is implicitly written as
\begin{align}
m^{t+1} &= m^{t} - \eta_{m}^{t} \delta m^{t},  &\delta m^{t} &= c C^{t} A m^{t} \label{eq:m-quad}\\
C^{t+1} &= C^{t} - \eta_{C}^{t} \delta C^{t}, &\delta C^{t} &= c C^{t} A C^{t},\label{eq:c-quad}
\end{align}
where $c > 0$ is the proportionality constant appearing in the proof of Proposition~\ref{prop:ng}.

In the following, we work on the following assumption: There are $\etammin > 0$ and $\etacmin > 0$ such that
\begin{gather}
\etammin \leq \eta_{m}^{t} \lambda_{1}( (C^{t})^{-1} \delta C^{t}) \leq 1,\label{eq:asmm}\\
\etacmin \leq \eta_{C}^{t} \lambda_{1}( (C^{t})^{-1} \delta C^{t}) \leq 1/2.\label{eq:asmc}
\end{gather}
These assumptions are satisfied, for example, if $\eta_{m}^{t}$ and $\eta_{C}^{t}$ are set for each iteration so that $\eta_{m}^{t} = \eta_{C}^{t} = \alpha / \lambda_{1}((C^{t})^{-1} \delta C^{t})$. In this case the natural gradient can be considered to be normalized by $\lambda_{1}((C^{t})^{-1} \delta C^{t})$ and the pseudo-learning rate is $\alpha$.

The next theorem states that the covariance matrix converges proportionally to the inverse of the Hessian matrix.

\begin{theorem}\label{thm:cond}
Assume \eqref{eq:asmc}. The condition number of $C^{t} A$ converges to one with the rate of convergence 
\begin{equation}
\limsup_{t\to\infty} \frac{\Cond(C^{t+1} A) - 1}{\Cond(C^{t} A) - 1} \leq \frac{1 - 2\etacmin}{1 - \etacmin}. \label{eq:rate-cond}
\end{equation}
Moreover, we have an upper bound 
\begin{equation}
\Cond(C^{t} A) \leq 1 + (1 - \etacmin)^{t}(\Cond(C^{0} A) - 1). \label{eq:upper-cond}
\end{equation}
If the limit $\etaclim = \lim_{t\to\infty} \eta_{C}^{t} \lambda_{1}((C^{t})^{-1} \delta C^{t})$ exists, $\etacmin$ is replaced with $\etaclim$ in \eqref{eq:rate-cond}.
\end{theorem}

\begin{proof}
Since $A$ is positive definite and symmetric, there exists the square root $\sqrt{A}$. Premultiplying and postmultiplying both side of covariance matrix update \eqref{eq:c-quad} by $\sqrt{c A}$, we have
\begin{equation*}
 c \sqrt{A} C^{t+1} \sqrt{A} = c \sqrt{A} C^t \sqrt{A} - \eta_C^{t} (c \sqrt{A} C^t \sqrt{A})^{2}.
\end{equation*}
Since $c \sqrt{A} C^{t} \sqrt{A}$ is positive definite and symmetric, there exists an eigenvalue decomposition $Q^{t} \Lambda^t (Q^t)^\trans$, where the diagonal elements of $\Lambda^t = \diag(\lambda_{1}^{t}, \dots, \lambda_{d}^{t})$ are the eigenvalues of $c \sqrt{A} C^{t} \sqrt{A}$ and each column of $Q^{t}$ is the eigenvector corresponding to each diagonal element of $\Lambda^t$. Then,
\begin{equation*}
 c \sqrt{A} C^{t+1} \sqrt{A} = Q^{t} \left(\Lambda^t - \eta_C^{t} (\Lambda^t)^{2}\right) (Q^{t})^\trans.
\end{equation*}
This means, $Q^t$ also diagonalizes $c \sqrt{A} C^{t+1} \sqrt{A}$. By mathematical induction we have that an orthogonal matrix $Q$ which diagonalizes $c \sqrt{A} C^{0} \sqrt{A}$ diagonalizes $c \sqrt{A} C^{t} \sqrt{A}$ for any $t \geq 0$ and we have
\begin{equation}
\Lambda^{t+1} = \Lambda^t - \eta_C^{t} (\Lambda^t)^{2}.\label{eq:diag-up}
\end{equation}

Next, we show that the condition number of $\Lambda^t$ converses to $1$ as $t \to \infty$. Remember $\delta C^{t} = c C^{t} A C^{t}$. We have $\lambda_{1}((C^{t})^{-1} \delta C^{t}) = \lambda_{1}(c A C^{t}) = \lambda_{1}(c \sqrt{A} C^{t} \sqrt{A}) = \lambda_{1}(\Lambda^{t})$. Then, by assumption~\eqref{eq:asmc} we have 
\begin{equation}
\etacmin \leq \eta_{C}^{t} \lambda_{1}(\Lambda^{t}) \leq 1/2.\label{eq:asmc2}
\end{equation} 
Moreover, since $\lambda_{1}(\Lambda^{t}) \geq \lambda_{i}^{t}$ for any $i$, we have $\eta_{C}^{t}(\lambda_{i}^{t} + \lambda_{j}^{t}) \leq 1$ for any $i$, $j$.

Suppose $\lambda_i^{t} \geq \lambda_j^{t}$ without loss of generality. From \eqref{eq:diag-up} and the inequality $\eta_{C}^{t} (\lambda_{i}^{t} + \lambda_{j}^{t}) \leq 1$, we have
\begin{equation*}
\begin{split}
\lambda_{i}^{t+1} - \lambda_{j}^{t+1} &= \lambda_i^{t}(1-\eta_{C}^{t} \lambda_i^{t}) - \lambda_j^{t}(1-\eta_{C}^{t} \lambda_j^{t})\\
 &= (1 - \underbrace{ \eta_{C}^{t} (\lambda_{i}^{t} + \lambda_{j}^{t})}_{\leq 1})(\underbrace{\lambda_{i}^{t}-\lambda_{j}^{t}}_{\geq 0})\geq 0
\end{split}
\end{equation*}
with equality holding if and only if $\lambda_{i}^{t}=\lambda_{j}^{t}$. Therefore, if $\lambda_i^{t} > \lambda_j^{t}$, then $\lambda_i^{t+1} > \lambda_j^{t+1}$, which implies that if $i$th and $j$th diagonal elements of $\Lambda^{0}$ are the maximum and minimum elements, $i$th and $j$th elements of $\Lambda^{t}$ are also the maximum and minimum elements of $\Lambda^{t}$. Without loss of generality we suppose $\lambda_{i}^{t} \geq \lambda_{j}^{t}$ for any $i \leq j$ for all $t \geq 0$. Then, $\lambda_{1}^{t}/\lambda_{d}^{t} = \Cond(\Lambda^{t}) = \Cond(C^{t} A)$. According to \eqref{eq:diag-up} we have
\begin{align}
\underbrace{\frac{\lambda_{1}^{t+1} - \lambda_{d}^{t+1}}{\lambda_{d}^{t+1}}}_{\Cond(C^{t+1} A)-1} &= \frac{\lambda_{1}^{t}(1 - \eta_{C}^t\lambda_{1}^{t}) - \lambda_{d}^{t}(1 - \eta_{C}^t\lambda_{d}^{t})}{\lambda_{d}^{t}(1 - \eta_{C}^t\lambda_{d}^{t})} \notag\\
&= \underbrace{\frac{(\lambda_{1}^{t} - \lambda_{d}^{t})}{\lambda_{d}^{t}}}_{\Cond(C^{t}A)-1} \frac{1 - \eta_{C}^t(\lambda_{1}^{t} + \lambda_{d}^{t})}{(1 - \eta_{C}^t\lambda_{d}^{t})}.\label{eq:cond-eq3}
\end{align}
Since
\begin{equation*}
\frac{1 - \eta_{C}^t(\lambda_{1}^{t} + \lambda_{d}^{t})}{(1 - \eta_{C}^t\lambda_{d}^{t})} = 1 - \frac{\eta_{C}^t \lambda_{1}^{t}}{1 - \eta_{C}^t \lambda_{1}^{t} \Cond^{-1}(C^t A)} \leq 1 - \eta_{C}^t \lambda_{1}^{t}
\end{equation*}
and the right-most side of the above inequality is bounded by $1 - \etacmin$ because of \eqref{eq:asmc2}, we have
\begin{equation}
\frac{\Cond(C^{t+1} A) - 1}{\Cond(C^{t} A) - 1} \leq 1 - \etacmin. \label{eq:cond-im}
\end{equation}
This implies $\lim_{t\to\infty}\Cond(C^{t} A) = 1$ and \eqref{eq:upper-cond}. Moreover, since $\lim_{t\to\infty}\Cond(C^{t} A) = \lim_{t\to\infty} \lambda_1^{t}/ \lambda_d^{t} = 1$, we have from \eqref{eq:cond-eq3} that 
\begin{align*}
\limsup_{t \to \infty} \frac{\Cond(C^{t+1} A) - 1}{\Cond(C^{t} A) - 1} &= \limsup_{t\to\infty }\frac{1 - 2 \eta_{C}^{t}\lambda_{1}^{t}}{1 - \eta_{C}^{t}\lambda_{1}^{t}} \\
&\leq \frac{1 - 2\etacmin}{1 - \etacmin}.
\end{align*}
This proves \eqref{eq:rate-cond}.

If the limit $\etaclim$ exists, it is easy to see from \eqref{eq:cond-im} that $\etacmin$ can be replaced with $\etaclim$ in \eqref{eq:rate-cond}. This completes the proof.
\end{proof}

Note that if $\eta_{C}^{t} = \alpha /\lambda_{1}((C^{t})^{-1} \delta C^{t})$ and $\alpha \leq 1/2$, we have that $\etaclim = \etacmin = \alpha$. We have from \eqref{eq:cond-eq3} that
\begin{equation}
\Cond(C^{t+1}A) = \Cond(C^{t}A) \frac{1-\alpha}{1 - \alpha\Cond^{-1}(C^{t}A)}
\label{eq:cond-sp}
\end{equation}
and the rate of convergence becomes $(1 - 2\alpha)/(1 - \alpha)$.

The next theorem states the global convergence of $m$ and $C$ and the speed of the convergence. In the following, we let $\norm{M}$ denote the Frobenius norm of $M$, namely $\norm{M} = \Tr^{1/2}(M^\trans M)$.

\begin{theorem}\label{thm:rate}
Assume \eqref{eq:asmc} and \eqref{eq:asmm}. Then, $\norm{m^{t}}$ and $\norm{C^{t}}$ converge to zero with the rate of convergence 
\begin{equation}
\limsup \frac{\norm{\kappa^{t+1}}}{\norm{\kappa^t}} \leq 1 - \etakmin, \label{eq:rate-k}
\end{equation}
where $\kappa$ is either $m$ or $C$ and $\kappa^{t}$ is either $m^{t}$ or $C^{t}$. If the limit $\etaklim = \lim_{t\to\infty} \eta_{\kappa}^{t} \lambda_{1}((C^{t})^{-1} \delta C^{t})$ exists, $\etakmin$ is replaced with $\etaklim$ in \eqref{eq:rate-k}.
\end{theorem}
\begin{proof}
Let $\sigma_{i}(\cdot)$ denote the $i$th largest singular value of the argument matrix. According to J.~von Neumann's trace inequality \cite{Mirsky1975mfm} we have $\abs{\Tr(M_{1} M_{2})} \leq \sum_{i=1}^{d} \sigma_{i}(M_{1})\sigma_{i}(M_{2})  \leq \sigma_{1}(M_{1}) \sum_{i=1}^{d} \sigma_{i}(M_{2})$, where $M_{1}$ and $M_{2}$ are any matrices in $\R^{d \times d}$. Let $M \in \R^{d \times d}$ be nonnegative definite and $S \in \R^{d\times d}$ is nonnegative definite symmetric. From the above inequality, we have
\begin{equation*}
\begin{split}
\norm{M S}^{2} &= \Tr(S M^\trans M S) = \Tr(M^\trans M S^{2})\\
 &\leq \sigma_{1}(M^\trans M) \sum_{i=1}^{d} \sigma_{i}(S^{2}) = \sigma_{1}^{2}(M) \norm{S}^{2}.
\end{split}
\end{equation*}
Applying this matrix norm inequality and the vector norm inequality $\norm{M x}^{2} \leq \sigma_{1}(M)^{2} \norm{x}^{2}$ to \eqref{eq:m-quad} and \eqref{eq:c-quad}, we have
\begin{equation*}
\frac{\norm{\kappa^{t+1}}}{\norm{\kappa^t}} \leq \sigma_{1}(I - c \eta_\kappa^{t} C^{t} A).
\end{equation*}
In light of Theorem~\ref{thm:cond}, we have that $\lim_{t \to \infty} C^{t} A / \lambda_{1}(C^{t}A) = I_{d}$. Then, from the assumptions \eqref{eq:asmm} and \eqref{eq:asmc} we have
\begin{equation*}
\begin{split}
&\limsup_{t} \sigma_{1}(I - c \eta_\kappa^{t} C^{t}A)\\
&= \limsup_{t}\sigma_{1}\biggl(I - \underbrace{\eta_\kappa^{t} \lambda_{1}(c C^{t}A)}_{\geq \etakmin} \underbrace{\frac{c C^{t}A}{\lambda_{1}(c C^{t}A)}}_{\to I_{d}} \biggr) \leq 1 - \etakmin.
\end{split}
\end{equation*}
This implies linear convergence of $\kappa$ with rate of convergence at most $1 - \etakmin$. 

If the limit $\etaklim$ exists, we can easily see from the above inequality that $\etakmin$ is replaced with $\etaklim$ in \eqref{eq:rate-k}. This ends the proof.
\end{proof}

Now we can see the importance of the covariance matrix adaptation quantitatively. Let $\eta_{m}^{t} = \eta_{C}^{t} = \alpha /\lambda_{1}((C^{t})^{-1} \delta C^{t})$. Then, the covariance matrix becomes proportional to the inverse of the Hessian at the speed given by \eqref{eq:cond-sp} and the rate of convergence of the parameter becomes $1 - \alpha$. Meanwhile, if the covariance matrix is restricted to a product of a scalar $v^{t}$ and the identity matrix, $C^{t} = v^{t} I$, then the rate of convergence is in $[1 - \alpha, 1 - \alpha \Cond^{-1}(A)]$\footnote{\footnotesize If the covariance matrix is restricted to a diagonal matrix, the target matrix is $\diag(A) = \diag(A_{1,1}, \dots, A_{d,d})$, i.e.\ $\lim_{t} \Cond(C^{t}\diag(A)) = 1$. the rate of convergence is in $[\sigma_{d}(I - \alpha \Cond(\tilde A)), \sigma_{1}(I - \alpha \Cond(\tilde A))]$, where $\tilde A = \diag(A)^{-1} A$. We omit the derivation due to the space limitation.}. Therefore, the rate of convergence becomes close to one in the worst case if $\Cond(A) \gg 1$.

From Theorem~\ref{thm:cond} we know that the deterministic NGD algorithm learns the inverse of the Hessian. The convergence of the covariance matrix to the inverse of the Hessian matrix in the CMA-ES has been anticipated but it has not been proven. Theorem~\ref{thm:cond} demonstrates this anticipation affirmatively for the deterministic NGD algorithm. Theorem~\ref{thm:rate} exhibits the linear convergence of the parameters. This implies that the rate of convergence of the expected objective value $J(\theta) \propto m^{\trans} A m + \Tr(C A)$ is also linear and equals to the rate of convergence of $\norm{C^{t}}$.

\section{Numerical simulation}\label{sec:exp}

The results in the previous section are for the deterministic (ideal) NGD algorithm. Thanks to Proposition~\ref{prop:consistency} we can expect that the stochastic NGD algorithm proposed in Section~\ref{sec:stoc} approximates the deterministic one arbitrarily close as the sample size $n$ is taken sufficiently large. In this section, we evaluate how well the stochastic variant with finite sample size approximates the deterministic one on a quadratic function. 

We consider the $20$-dimensional ellipsoid function 
\begin{equation*}
f(x) = \sum_{i=1}^{d} 10^{\frac{6(i-1)}{d - 1}} x_{i}^{2} \qquad (d = 20).
\end{equation*}
Note that the ellipsoid function is separable and convex but our algorithm does not exploit the separability and convexity. The eigenvalues (diagonal elements) of the Hessian matrix of the ellipsoid function range in $[1, 10^{6}]$. We set the initial parameters as $m^{0} = (0,\dots,0)^\trans$ and $C^{0} = I_{d}$.

We design the learning rates as 
\begin{equation*}
\eta_{m}^{t} = \frac{1}{\sigma_{1}(Z^{t})} \text{ and } \eta_{C}^{t} = \frac{c_{C}}{2\sigma_{1}(Z^{t})}, \quad c_{C} \leq 1.
\end{equation*}
Here, $Z^{t}$ is a matrix defined in \eqref{eq:learning}. 

\subsection{Effect of Sample Size and Learning Rate}

\newcommand{\figlambdasize}{0.78\hsize}
\begin{figure}[!t]
\centering
\begin{tabular}{c}
\includegraphics[width=\figlambdasize]{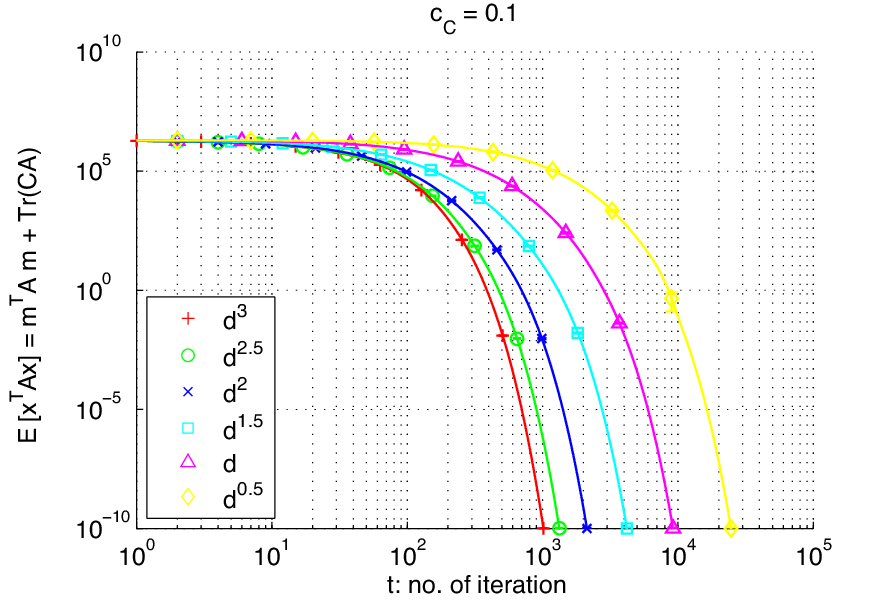}\\
\includegraphics[width=\figlambdasize]{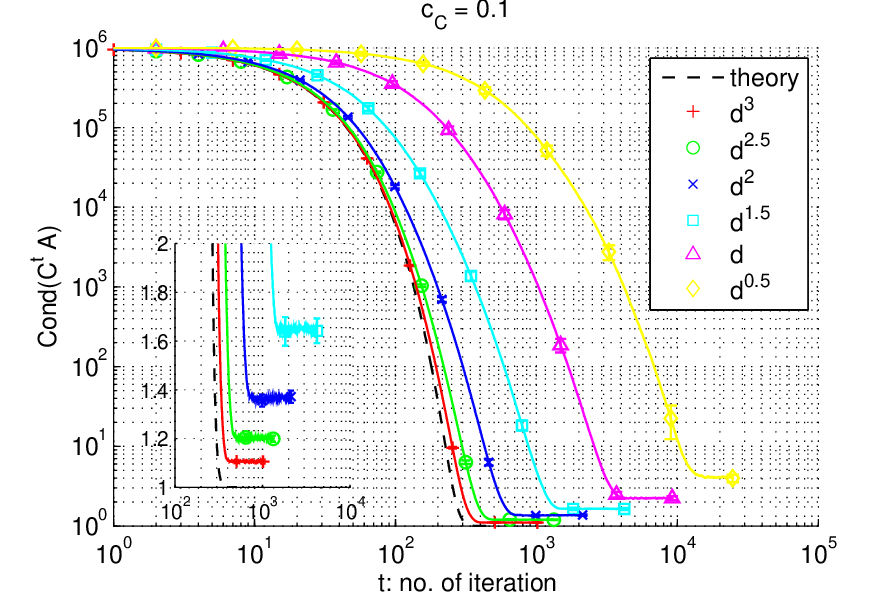}\\
\includegraphics[width=\figlambdasize]{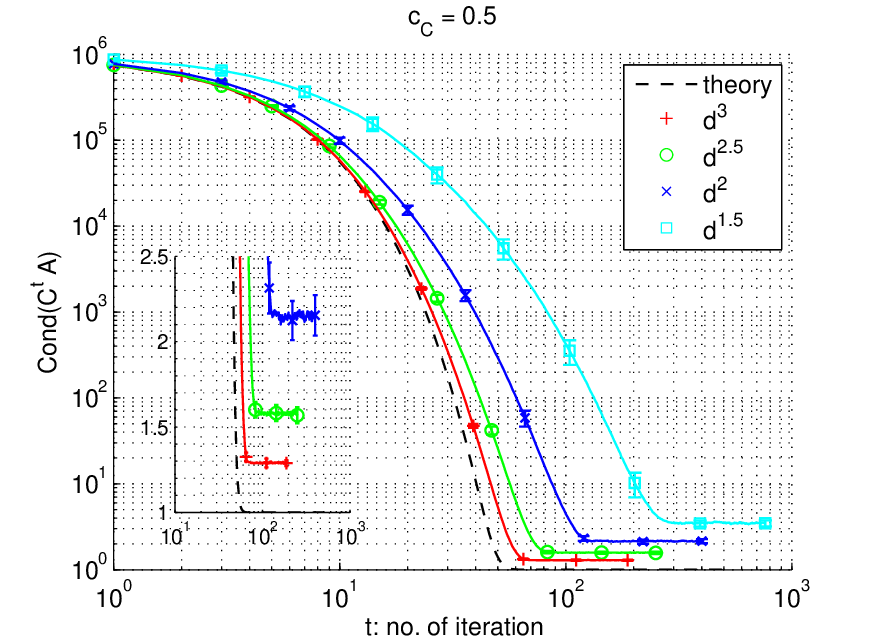}\\
\includegraphics[width=\figlambdasize]{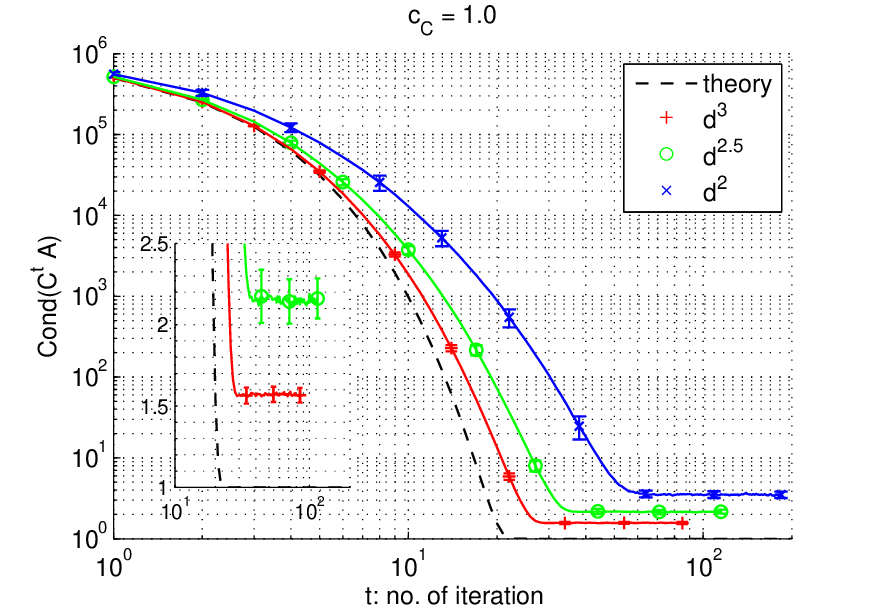}
\end{tabular}
\caption{Averages and standard deviations of the condition numbers $\Cond(C^{t}A)$ for $c_{C} = 0.1$, $0.5$, and $1.0$ and the expected objective function values for $c_{C} = 0.1$. Theoretical curves \eqref{eq:cond-sp} of the condition number are also illustrated with dashed lines. All the lines are cut after first reach of the expected objective value to $10^{-10}$. Some results are omitted, for example $n = d$ for $c_{C} = 0.5$, because numerically unstable computation occurs during search.}
\label{fig:lambda}
\end{figure}

First, we investigate the effect of the sample size $n$ and the coefficient $c_{C}$ of the learning rate $\eta_{C}^{t}$. We try the following sample sizes: $\lceil d^{1/2} \rceil$, $d$, $\lceil d^{3/2} \rceil$, $d^2$, $\lceil d^{5/2} \rceil$, $d^{3}$.

Figure~\ref{fig:lambda} illustrates the slope of the condition number $\Cond(C^{t} A)$ and the theoretical curve \eqref{eq:cond-sp} and the slope of the expected objective function value $\E_{X\sim P_{\theta^{t}}}[X^\trans A X] = (m^{t})^\trans A m^{t} + \Tr(C^{t}A)$, averaged over $50$ independent trials. When the sample size is larger, we see the closer performance to the theoretical result. When $n = d^{3}$ and $c_{C} = 0.1$, the convergence curve of the condition number approximated well the theoretical curve and the final condition number is $\Cond(C^{t} A) \approx 1.1$. When $n = \lceil d^{1/2} \rceil$ and $c_{C} = 0.1$, it takes more than $30$ times longer to learn the covariance matrix and the final condition number becomes $\Cond(C^{t} A) \approx 4.0$, although the stochastic algorithm still works successfully. We attain a little higher condition numbers when we choose larger learning rates $c_{C} = 0.5$, $1.0$. For example, the final condition numbers are $\Cond(C^{t} A) \approx 1.3$ for $n = d^{3}$ and $c_{C} = 0.5$, and $\Cond(C^{t} A) \approx 1.6$ for $n = d^{3}$ and $c_{C} = 1.0$. This is because smaller learning rates have more effect of averaging the natural gradient estimates over iterations and reducing the estimation variance.

Note that we observe a slightly slower adaptation of the covariance matrix at the beginning in case that we set $m^{0} = (10, \dots, 10)$, although the adaptation behavior \eqref{eq:cond-sp} does not change in theory. See Figure~\ref{fig:m}. This attributes to the estimation precision of $\Vfh$. If the squared Mahalanobis distance $(m^{t})^\trans (C^{t})^{-1} m^{t}$ between the origin (the global optimum) and the current mean with respect to $C^{t}$ is larger, the function landscape around $m^{t}$ looks more like linear function. Then $\Vfh(x_{i})$ are far from the exact values, especially in case a small sample size is chosen.

\begin{figure}[t]
\centering
\includegraphics[width=0.9\hsize]{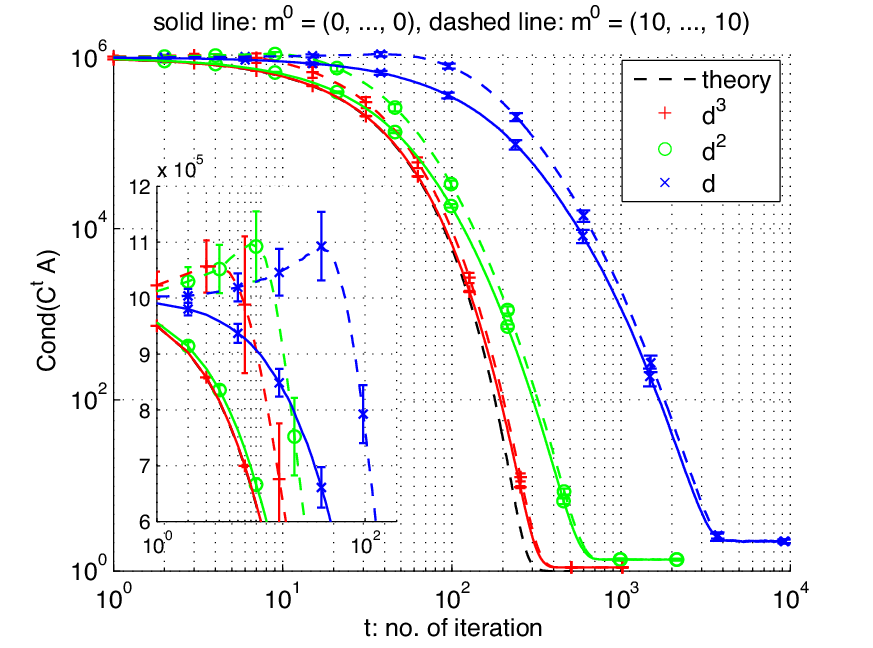}
\caption{Averages and standard deviations of the change of the condition numbers for $c_{C} = 0.1$ for $n = d$, $d^{2}$, $d^{3}$ with initial mean vector $m = (0, \dots, 0)$ or $m = (10, \dots, 10)$. Other settings are the same as in Figure~\ref{fig:lambda}.}
\label{fig:m}
\end{figure}

\subsection{Comparison with Rank-$\mu$ update CMA-ES}

Finally, we study how well this stochastic algorithm simulates the CMA-ES. We test the pure rank-$\mu$ update CMA-ES with weight scheme \eqref{eq:intermediate}. We set the learning rates following \cite{Hansen2004ppsn}
\begin{equation*}
\eta_{m}^{t} = 1, \qquad \eta_{C}^{t} = \frac{2\muw - 1}{(d + 2)^{2} + \muw},
\end{equation*}
where $\muw = 1 / \sum_{i=1}^{n} w_{i}^{2}$. We choose $c_{C}$ for our algorithm so that the speed of adaptation for each model is almost the same.

Figure~\ref{fig:cma} shows the results for each method for $n = d$ and $n = d^2$. In both case, we confirm similar behaviors of the pure rank-$\mu$ update CMA-ES and our algorithm despite their dissimilar weight-value settings. The similar change of performance illustrated in Figure~\ref{fig:m} is also observed for the pure rank-$\mu$ update CMA-ES. From this result, we conclude that it is possible to estimate the performance of the pure rank-$\mu$ update CMA-ES by our natural gradient algorithm, which is theoretically more attractive. 

However, note that the pure rank-$\mu$ update CMA-ES is not the standard CMA-ES \cite{Hansen2004ppsn} and the standard CMA-ES performs better than the pure rank-$\mu$ update CMA-ES. The standard CMA-ES employes so-called evolution paths to adapt the covariance matrix and the global scale of the covariance matrix, which is called step-size in the CMA-ES context. Moreover, the standard CMA-ES employes weighted recombination, where different values are assigned to the weights for $R_{i} \leq \lfloor n/2 \rfloor$, which is only slightly better than intermediate recombination \eqref{eq:intermediate} and even similar to our setting. Furthermore, the similar performance observed is only on a quadratic function. If there are certain functions which distinguish our algorithm from the (rank-$\mu$) CMA-ES, this may help to understand both the NGD algorithm and the CMA-ES. Further study on these topics is required.

\newcommand{\figcmasize}{0.48\hsize}
\begin{figure}[t]
\begin{tabular}{cc}
\begin{minipage}{\figcmasize}
\centering
\includegraphics[width=\hsize]{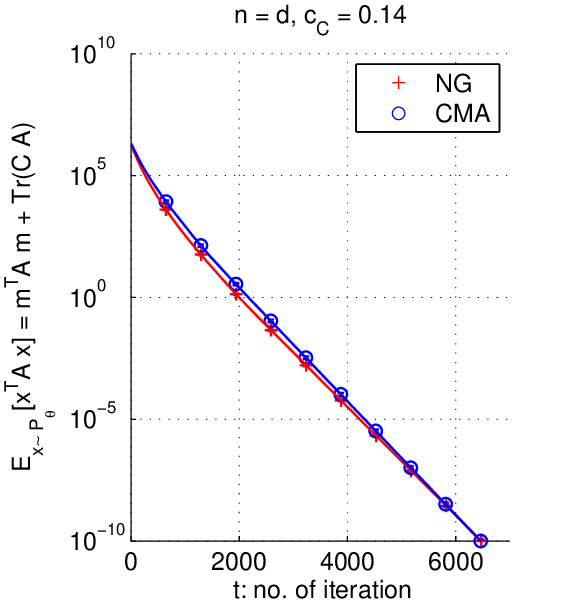}
\end{minipage}
\begin{minipage}{\figcmasize}
\centering
\includegraphics[width=\hsize]{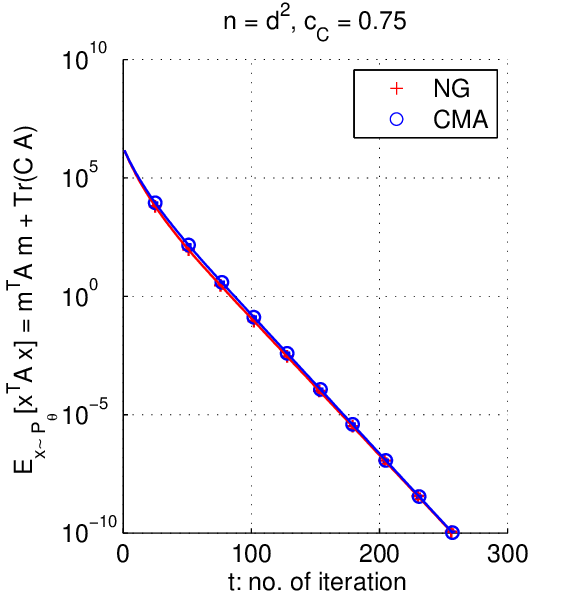}\\
\end{minipage}
\\
\begin{minipage}{\figcmasize}
\centering
\includegraphics[width=\hsize]{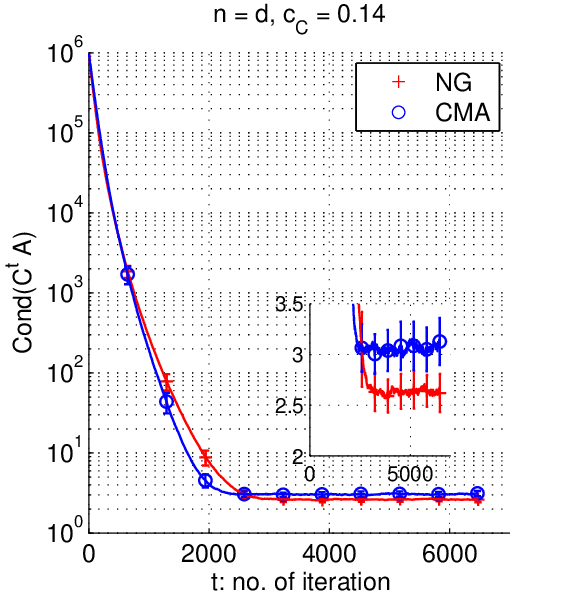}\\
\end{minipage}
\begin{minipage}{\figcmasize}
\centering
\includegraphics[width=\hsize]{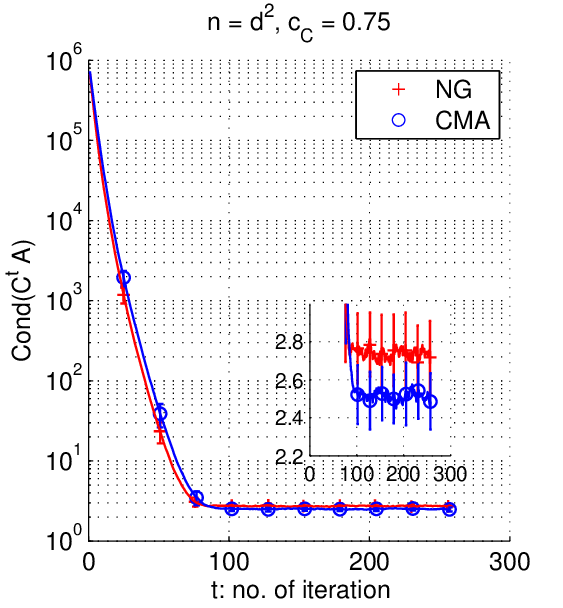}\\
\end{minipage}
\end{tabular}
\caption{Averages and standard deviations of the change of the condition numbers and the change of the expected objective function values for $n = d$ and $c_{C} = 0.14$ on the left, and for $n = d^{2}$ and $c_{C} = 0.75$ on the right. Other settings are the same as in Figure~\ref{fig:lambda}.}
\label{fig:cma}
\end{figure}

\section{Conclusion}\label{sec:conc}

We have proposed a novel natural gradient descent (NGD) method where the objective function is transformed to a function defined on the parameter space of probability distributions. We have proven that the deterministic NGD method learns the inverse of the Hessian of the original objective function that is any monotonic convex-quadratic-composite function. Linear convergence of the mean vector and the covariance matrix has been also proven. The numerical results for the stochastic NGD algorithm have shown that the stochastic algorithm approximates the deterministic one well when the sample size is sufficiently large. Moreover, we have confirmed that the stochastic NGD algorithm and the pure rank-$\mu$ update CMA-ES behave very similarly on a quadratic function. 

The contribution of the paper is to derive a novel NGD algorithm that can be viewed as a variant of the CMA-ES from the first principle of information geometry. This allows us to analyze the algorithm theoretically. Our theoretical results in Section~\ref{sec:conv} imply that there is at least one weight-value setting in the CMA-ES such that the covariance matrix learns the inverse of the Hessian of the objective function. Moreover, since our algorithm does not only share most of the important properties of the rank-$\mu$ update CMA-ES, but also is confirmed to perform similarly to the pure rank-$\mu$ update CMA-ES on a quadratic function by numerical simulations, we could study our algorithm to find out limitations of the pure rank-$\mu$ update CMA-ES and to discover a way to improve the CMA-ES. 


%


%
%
\end{document}